\documentclass[11pt]{article}

\usepackage{jmlr2e}

\usepackage{amsmath}

\usepackage[]{algorithm2e}

\usepackage[utf8]{inputenc}
\usepackage[english]{babel}
\usepackage{dirtytalk}

\renewcommand{\H}{\mathcal{H}}

\newcommand{\G}{\mathcal{G}}
\newcommand{\A}{\mathcal{A}}
\newcommand{\B}{\mathcal{B}}

\newcommand{\Y}{\mathcal{Y}}
\newcommand{\Z}{\mathcal{Z}}
\newcommand{\E}{\mathop{\mathbb{E}}}
\newcommand{\R}{\mathbb{R}}

\renewcommand{\H}{\mathcal{H}}

\renewcommand{\L}{\mathcal{L}}

\graphicspath{ {figures/} }

\author{Michael Rabadi \\
Spotify \\
New York, NY 10011 \\
\texttt{mrabadi@spotify.com} \\
}

\ShortHeadings{Co-adaptive learning over a countable space}{Rabadi}
\firstpageno{1}

\begin{document}

\title{Co-adaptive learning over a countable space}

\author{\name Michael Rabadi \email mrabadi@spotify.com \\
       \addr Spotify, Inc.\\
       New York, NY, USA}

\editor{}

\maketitle

\begin{abstract}%
Co-adaptation is a special form of on-line learning where an algorithm $\A$ must 
assist an unknown algorithm $\B$ to perform some task. This is a general 
framework and has applications in recommendation systems, search, education, and 
much more. Today, the most common use of co-adaptive algorithms is in 
brain-computer interfacing (BCI), where algorithms help patients gain and 
maintain control over prosthetic devices. While previous studies have shown 
strong empirical results \cite{kowalski2013, orsborn2014} or have been studied 
in specific examples \cite{merel2013, merel2015}, there is no general analysis 
of the co-adaptive learning problem. Here we will study the co-adaptive learning 
problem in the online, closed-loop setting. We will prove that, with high 
probability, co-adaptive learning is guaranteed to outperform learning with a 
fixed decoder as long as a particular condition is met.
\end{abstract}


\section{Introduction} \label{sec:introduction}

There are many circumstances where an algorithm $\B$ must learn to perform a 
task while operating through some unknown algorithm $\A$. One common example is 
search: a person $\B$ must learn how to provide queries to a search engine $\A$ 
in order to get good search results. In older search engines, $\A$ was a fixed 
function. These days $\A$ tries to help $\B$ find better results using a history 
of $\B$'s queries and selections. While this isn't traditionally thought of as 
co-adaptive learning, experience shows that both the user $\B$ and the search 
engine $\A$ learn to work together to provide $\B$ with good results.

Co-adaptation is more explicitly studied in the context of brain-computer 
interfacing (BCI). BCIs are systems that help paralyzed patients gain control 
over a computer or a prosthetic by decoding their brain activity. In older 
experiments, it has been shown that patients $\B$ were able to learn to perform 
a task through a fixed decoder $\A$, which mapped a patient's brain signals to a 
computer cursor \cite{carmena2003, ganguly2009}. More recently, some 
experimental and theoretical frameworks have been used to examine how decoders 
should adapt in order to improve performance \cite{digiovanna2009, kowalski2013, 
dangi2014}. However, most of these previous studies focused on how $\A$ should 
learn without considering how $\B$ might adapt. A notable exception is 
\cite{merel2015}, where, using some strong assumptions, they were able to show 
that co-adaptive learning was at least as good as learning with a fixed decoder. 
Unfortunately, they had to assume that $\B$'s loss function was known to $\A$ 
and, in particular, was the mean squared loss. Those assumptions may be 
reasonable for some BCI tasks, but do not hold in general.

Here we analyze the co-adaptation learning problem, which is summarized in 
Figure \ref{fig:coadapt}. In this setting, $\B$ has some intention $y_t$ at time 
$t$. $\B$ selects an encoder $g_t \in \G$ with the intention of making the 
input 
$\hat{y}_{t-1}$ close to the intended output. However, $\B$ must work through 
the decoding algorithm $\A$, which selects a hypothesis $h_t \in \H$ with the 
intention of producing a $\hat{y_t}$ close to $y_t$. Previous studies have 
assumed that $\A$ has 
knowledge of $y_t$ during training, or at least can infer $y_t$. However this is 
usually an unreasonable assumption. Instead, we analyze the case where $\A$ does 
not have access to $y_t$. A natural first question: \say{Is co-adaptive 
learning any better than learning with a fixed decoder?} We will show that, 
under some mild assumptions, co-adaptive learning can indeed outperform 
learning with a fixed decoder.

\begin{figure}[h]
    \centering
    \includegraphics[width=0.4\textwidth]{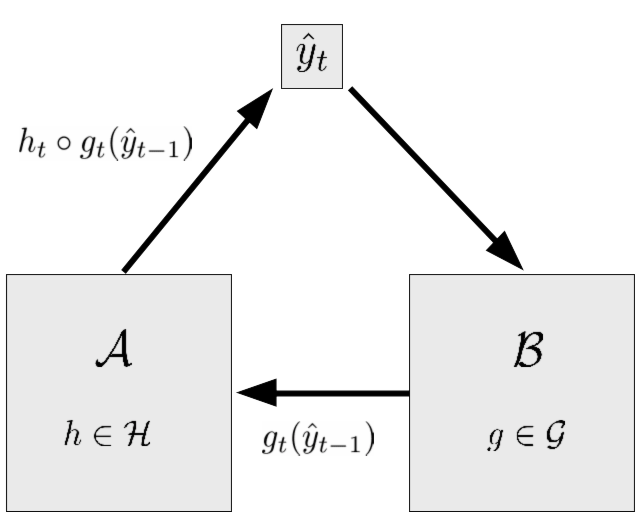}
    \caption{The co-adaptive learning scenario. }
    \label{fig:coadapt}
\end{figure}

\subsection{Formalization of the learning problem} \label{sec:formalizing}

Let $\G$ be the hypothesis set of the encoder $\B$, such that $\G = \{g: \Y 
\to \Z\}$, where $\Y$ is some countable space and $\Z$ is some arbitrary 
feature space. In practice, $\Z$ is typically a Hilbert space. $\H$ will denote 
the hypothesis set for the decoder $\A$, where $\H = \{h: \Z \to \Y \}$.
As outlined above, we wish to find an algorithm $\A$ that can help another 
algorithm $\B$ minimize its loss. Specifically, $\B$ has a sequence of 
intentions $(y_1, ..., y_T)$ where $y_t \in \Y$. $\B$ encodes some control $z_t 
= g_t(\hat{y}_{t-1}), z_t \in \Z$ that it must pass through $\A$. $\A$ then 
selects a decoder $h_t \in \H$ and produces an output $\hat{y} = h_t  \circ 
g_t(\hat{y}_{t-1})$, where $h \in \H$.

In the simplest case, $\A$ is a fixed decoder and all of the learning is done 
by $\B$, which must select a mapping $g_t \in \G$ at each time point $t$. 
Therefore, we wish to evaluate our co-adaptive algorithm with respect to a fixed 
decoder $\tilde{h} \in \H$. Thus, we wish to evaluate our model in terms of 
regret. Specifically, $\A$ should choose a hypothesis $h_t \in \H$ that 
minimizes 
its regret with respect to a fixed decoder in hindsight:

$$
R_T =  \sum_{t=1}^T L_\B(h_t \circ g_t( \hat{y}_{t-1}), y_t) - \min_{g^\prime 
\in \G} \sum_{t=1}^T L_\B(\tilde{h} \circ g_t^\prime (\tilde{y}_{t-1}), y_t)
$$

where $L_\B$ is the loss function of $\B$ and $\tilde{y}_t$ is the predicted 
value of $y_t$ when using a fixed decoder. Thus, co-adaptive learning is only 
better than a fixed decoder when $\sum_{t=1}^T L_\B(h_t \circ g_t( 
\hat{y}_{t-1}, y_t) <  \min_{g \in \G} \sum_{t=1}^T L_\B(\tilde{h} \circ 
g_t^\prime 
(\hat{y}_{t-1}), y_t)$ and so $R_T < 0$.

In Section \ref{sec:theory} we will 
show that, under some mild assumptions, we can prove that coadaptation 
outperforms a fixed decoder with high probability as long as a particular 
condition is met.

\subsection{Assumptions} \label{sec:assumptions}

Our analysis will rely on a couple of mild assumptions:

\begin{enumerate}
	\item $\Y$ is a countable space.
	\item $L_\B$ is Lipschitz continuous with Lipschitz constant $\ell_\B > 
0$ with respect to the Hamming metric (explained below).
\end{enumerate}

Assumption 1 is a straightforward assumption and applies to many problems, 
including search and the BCI scenario (assuming finite output). The second 
assumption is necessary for deriving our main theorem. We will 
discuss this assumption more in the next section. 

\subsection{Preliminaries} \label{sec:prelim}

We will need the following theorem in order to prove the main theorem of this 
paper. For ease of notation, let $\psi(y) = \sum_{t=1}^T L_\B(\tilde{h} \circ 
g_t(\tilde{y}_{t-1}), y_t)$. Furthermore, set 

$$
\bar{\eta}_{ij} \equiv \sup_{\substack{y^{i-1} \in \Y^{i-1}, w, \hat{w} \in \Y 
\\ \Pr[y^i = (y^{i-1},w)] > 0, \Pr[y^i = (y^{i-1},\hat{w})]>0}} \| \L(Y_j^n | 
Y^i = (y^{i-1},w)) - \L(Y_j^n|Y^i=(y^{i-1},\hat{w}))\|_{\mathrm{TV}},
$$

where $\| A - B \|_{\mathrm{TV}}$ is the total variation distance between 
probability distributions $A$ and $B$ and where $\L(Z|Y=y)$ is the conditional 
distribution of $Z$ 
given $Y=y$. We will then define $M_T$ to be the largest sum of variational 
distances over all sub-sequences in $\Y^T$:

$$
M_T \equiv \max_{1 \leq t \leq T} ( 1 + \bar{\eta}_{t, t+1} + \cdots + 
\bar{\eta}_{t,T} )
$$

We next define the Hamming metric $d_H(x, y) = \sum_{t=1}^T 1_{x_t \neq y_t}$, 
where $1_{x_t \neq y_t}$ is the indicator function. Thus, the Hamming metric is 
a count of all points that differ between two sequences. In our scenario, the 
Hamming metric is used between sequences $y$ and $\tilde{y}$. We only need to 
evaluate the sequences with respect to $L_\B$, where 
$L_\B$ is $\ell_\B$-Lipschitz with respect to the Hamming metric. Specifically, 
we require that $| \psi(S) - \psi(S^\prime) | \leq \ell_\B d_H(S, 
S^\prime)$, where $S$ and $S^\prime$ are two sequences that differ by a single 
point.

With these assumptions, we can give Theorem \ref{thm:kontorovich} from \cite{kontorovich2008}.

\begin{theorem}[\cite{kontorovich2008} 1.1]
\label{thm:kontorovich}
Without loss of generality, let $y = (y_1, ... y_T) \in \Y^T$ be a sequence of 
random variables, and $\psi : \Y^T \to \R$ is a $\ell_\B$-Lipschitz continuous 
function with respect to $d_H(\cdot, \cdot)$. Then, for any $\epsilon > 0$:

$$
\Pr[ | \psi - \E[\psi] | > \epsilon ] \leq 2 \exp \bigg( \frac{-\epsilon^2}{2 
T \ell_\B^2 M_T^2} \bigg)
$$
\end{theorem}

Theorem \ref{thm:kontorovich} is a version of McDiarmid's inequality over 
dependent random variables, which relaxes the i.i.d. assumption. We will 
use this theorem to give a probabilistic bound for the learning problem, which 
we will give and prove in the next section.

\section{Theoretical analysis of the co-adaptive learning problem} \label{sec:theory}

In this section we give the main theorem, which gives a condition that 
guarantees that co-adaptive learning will be better than learning with a fixed 
decoder with high probability. Recall that we will evaluate our algorithm with 
respect to regret:

$$
R_T =  \sum_{t=1}^T L_\B(h_t \circ g_t( \hat{y}_{t-1}), y_t) - \min_{g^\prime 
\in \G} \sum_{t=1}^T L_\B(\tilde{h} \circ g_t^\prime (\tilde{y}_{t-1}), y_t)
$$

Specifically, we can say that co-adaptive learning outperforms learning with a fixed decoder when $R_T <  0$. 

\begin{theorem}
\label{thm:main}
Given the assumptions from Section \ref{sec:assumptions} hold. Co-adaptation outperforms learning with a fixed decoder with probability at least $1 - \delta$ if the following holds:

$$
\sum_{t=1}^T L_\B (h_t \circ g_t(\hat{y}_{t-1}), y_t) + \ell_\B M_T \sqrt{2 T \log \frac{2}{\delta}} < \sum_{t=1}^T \varepsilon_t
$$

where $\varepsilon_t$ is the expected minimal loss at time $t$.

\end{theorem}

\begin{proof}
We want a probabilistic bound for $\psi \equiv \sum_{t=1}^T L_\B(\tilde{h} \circ 
g_t^\prime (\tilde{y}_{t-1}), y_t)$. We can immediately derive the following 
bound from Theorem \ref{thm:kontorovich}.

$$
\Pr\Big[ \Big| \E[\psi] - \psi \Big| > \epsilon\Big] \leq 2 \exp \bigg( 
\frac{-\epsilon^2}{2 T \ell_\B^2 M_T^2} \bigg)
$$

Thus, with probability at least $1- \delta$, the following holds:

\begin{align*}
\Big| \E [\psi] - \psi \Big| &\leq \ell_\B M_T \sqrt{2 T \log 
\frac{2}{\delta}} \\
\Longrightarrow \psi &\geq \E[\psi] - \ell_\B M_T \sqrt{2 T \log 
\frac{2}{\delta}} 
\end{align*}

We then take the $\inf_{g^\prime \in \G} \psi$, where $g^\prime$ is a sequence 
of $g_t^\prime$.

\begin{align*}
\inf_{g^\prime \in \G} \psi &\geq \inf_{g^\prime \in \G} \E [ \psi ] - 
\ell_\B M_T \sqrt{2T \log \frac{2}{\delta}} \\
&= \inf_{g^\prime \in \G} \E \Bigg[ \sum_{t=1}^T L_\B (\tilde{h} \circ 
g^\prime_t(\tilde{y}_{t-1}), y_t)\Bigg] - \ell_\B M_T \sqrt{2 T \log 
\frac{2}{\delta}} \\
&= \inf_{g^\prime \in \G} \sum_{t=1}^T \E [ L_\B (\tilde{h} \circ 
g^\prime_t(\tilde{y}_{t-1}), y_t) | Z] - \ell_\B M_T \sqrt{2 T \log 
\frac{2}{\delta}} \\
&\geq \sum_{t=1}^T \inf_{g_t^\prime \in \G} \E [ L_\B(\tilde{h} \circ 
g_t^\prime(\tilde{y}_{t-1}),y_t) | Z] - \ell_ \B M_T \sqrt{2 T \log 
\frac{2}{\delta}} \\
&= \sum_{t=1}^T \inf_{g_t^\prime \in \G} \sum_{\tilde{y}_t \in \Y} \sum_{y \in 
\Y} L_\B(\tilde{y}_t,y_t)\L\big(\tilde{y}_t, y_t | Z \big) - 
\ell_ \B M_T \sqrt{2 T \log 
\frac{2}{\delta}} \\
&= \sum_{t=1}^T \varepsilon_t - \ell_\B M_T \sqrt{2 T \log \frac{2}{\delta}}
\end{align*}
where $Z = \Big(\{g_{t-i}^\prime \}_{i 
\in [1,t]}, \{y_{t-i}\}_{i \in [1,t]}, \{\tilde{y}_{t-i}\}_{i \in [1,t]}, 
\tilde{h}\Big)$, $\L(\tilde{y}_t,y_t|Z)$ is the joint distribution function of 
$(\tilde{y}_t,y_t)$ conditioned 
on $Z$, and $\{z_{t-i}\}_{i \in [1,t]}$ is the sequence of previous points for 
some variable $z$, which is all that $\B$ has to learn with and, consequently, 
use 
to select some $g_t^\prime \in \G$.

The second equality holds by the conditional independence of the loss. The 
second inequality holds by the super-additivity of the $\min$ function. The 
third equality is the definition of the expectation of the loss function, which 
is a function over the joint distribution of $\hat{y}$ and $y$, where $\hat{y}$ 
is determined by selecting some $g_t^\prime \in \G$. The 
final equality is the definition of $\varepsilon_t$. We can then use this lower 
bound on $\inf_{g^\prime \in \G} \psi$ to impose the probabilistic 
constraint above.
\end{proof}

Theorem \ref{thm:main} gives us a constraint that guarantees that the 
co-adaptive learner will out-perform learning with a fixed decoder. The bound 
depends on properties of the conditional probabilities for the different 
sequences and the loss function used. The benefit of this bound is that we can 
use the empirical loss, $\sum_{t=1}^T L_\B (h_t \circ g_t(\tilde{y}_{t-1}), 
y_t)$ 
to guarantee that co-adaptive learning is better than learning with a fixed 
decoder with high probability. While, $\sum_{t=1}^T \varepsilon_t$ is not 
accessible in general, there may be cases where we can guarantee that, in 
expectation, there is some infimal error. This is certainly the case when either 
$\tilde{h}$ or $\G$ is not rich enough to perform some task with zero error. 
Indeed, it is easy to see that for all $\tilde{h}$ and $\G$, there exists some 
set of conditional distributions $\L(\tilde{y}_t, y_t|Z), \forall \tilde{y}, y 
\in \Y^T$ where the infimal error is greater than zero. Indeed this occurs 
whenever $\L(\tilde{y}_t = y_t | y_t, Z) < 1$ or $\L(y_t|Z) < 1$, which holds 
by 
the definition of the conditional joint distribution:

$$
\L(\tilde{y}_t, y_t | Z) = \L(\tilde{y}_t|y_t, Z) \L(y_t|Z)
$$

\section{Conclusion}

We gave a condition for the co-adaptive learning problem that, if verified, 
guarantees that co-adaptive learning outperforms learning with a fixed decoder 
with high probability. This work marks a significant generalization of the 
co-adaptive learning problem. Our analysis is more general than one with 
the standard 
i.i.d. assumption. However, one drawback is that the derived constraint is not 
accessible in 
general: we do not always know the expected minimal loss as a function 
of time for $\B$ learning over a fixed decoder. However, in practice we may be 
able to  approximate $\varepsilon_t$ by measuring the learning rate and initial 
transition probability.

\bibliography{mybib}

\end{document}